\documentclass{article}

\usepackage[final]{nips_2017}

\usepackage{appendix}
\usepackage[utf8]{inputenc} 
\usepackage[T1]{fontenc}    
\usepackage{hyperref}       
\usepackage{url}            
\usepackage{booktabs}       
\usepackage{amsfonts}       

\usepackage{times}
\usepackage{helvet}
\usepackage{courier}
\usepackage{amsmath, amsthm, amssymb}
\usepackage{amsmath}
\usepackage{hhline}
\usepackage{natbib}
\usepackage{epsfig}
\usepackage{epstopdf}
\usepackage{amssymb}

\usepackage{graphicx,wrapfig}
\usepackage{caption}
\usepackage{subcaption}

\usepackage[linesnumbered,noend,ruled]{algorithm2e}
\usepackage{graphicx,amssymb,amstext,amsmath, amsthm,setspace}

\usepackage{verbatim}
\newtheorem{theorem}{Theorem}

\newtheorem{lemma}{Lemma}

\newtheorem{proposition}{Proposition}

\newcommand{\squishlist}{
 \begin{list}{$\bullet$}
  { \setlength{\itemsep}{0pt}
     \setlength{\parsep}{3pt}
     \setlength{\topsep}{3pt}
     \setlength{\partopsep}{0pt}
     \setlength{\leftmargin}{1.5em}
     \setlength{\labelwidth}{1em}
     \setlength{\labelsep}{0.5em} } }

\newcommand{\squishlisttwo}{
 \begin{list}{$\bullet$}
  { \setlength{\itemsep}{0pt}
    \setlength{\parsep}{0pt}
    \setlength{	opsep}{0pt}
    \setlength{\partopsep}{0pt}
    \setlength{\leftmargin}{2em}
    \setlength{\labelwidth}{1.5em}
    \setlength{\labelsep}{0.5em} } }

\newcommand{\squishend}{
  \end{list}  }

\DeclareMathOperator{\n}{n}

\DeclareMathOperator{\a_}{A}
\DeclareMathOperator{\fa_}{fA}
\DeclareMathOperator{\c_}{C}
\DeclareMathOperator{\fc_}{fC}

\DeclareMathOperator{\ac}{AC}
\DeclareMathOperator{\fac}{fAC}
\DeclareMathOperator{\fafc}{fAfC}
\DeclareMathOperator{\afc}{AfC}

\renewcommand{\=}{\!=\!}

\newcommand{\bs}{\boldsymbol}

\newcommand{\bn}{\boldsymbol{n}}

\newcommand{\cmdp}{{$\mathbb{C}$Dec-POMDP}}

\title{Policy Gradient With Value Function Approximation \\ For Collective Multiagent Planning}
\author{
  Duc Thien Nguyen \;\;  Akshat Kumar\;\; Hoong Chuin Lau\\
  School of Information Systems\\
  Singapore Management University\\
  80 Stamford Road, Singapore 178902 \\
  \texttt{\{dtnguyen.2014,akshatkumar,hclau\}@smu.edu.sg} \\
}

\begin{document}

\maketitle

\begin{abstract}
Decentralized (PO)MDPs provide an expressive framework for sequential decision making in a multiagent system. Given their computational complexity, recent research has focused on tractable yet practical subclasses of Dec-POMDPs. We address such a subclass called {\cmdp} where the collective behavior of a population of agents affects the joint-reward and environment dynamics. Our main contribution is an actor-critic (AC) reinforcement learning method for optimizing {\cmdp} policies. Vanilla AC has slow convergence for larger problems. To address this, we show how a particular decomposition of the approximate action-value function over agents leads to effective updates, and also derive a new way to train the critic based on local reward signals. Comparisons on a synthetic benchmark and a real world taxi fleet optimization problem show that our new AC approach provides better quality solutions than previous best approaches.
\end{abstract}

\section{Introduction}

Decentralized partially observable MDPs (Dec-POMDPs) have emerged in recent years as a promising framework for multiagent collaborative sequential decision making~\citep{bern02}. Dec-POMDPs model settings where agents act based on different partial observations about the environment and each other to maximize a global objective. Applications of Dec-POMDPs include coordinating planetary rovers~\citep{Becker04JAIR}, multi-robot coordination~\citep{Amato15} and throughput optimization in wireless network~\citep{Winstein:2013,PajarinenHP14}. However, solving Dec-POMDPs is computationally challenging, being NEXP-Hard even for 2-agent problems~\citep{bern02}. 

To increase scalability and application to practical problems, past research has explored restricted interactions among agents such as state transition and observation independence~\citep{nair05,Kumar2011,kumar2015probabilistic}, event driven interactions~\citep{BZLaamas04} and weak coupling among agents~\citep{Witwicki10}. Recently, a number of works have focused on settings where agent identities do not affect interactions among agents. Instead, environment dynamics  are primarily driven by the \textit{collective} influence of agents~\citep{varakantham2014decentralized,sonu2015individual,robbel2016exploiting,NguyenKL17}, similar to well known congestion games~\citep{meyers12}. Several problems in urban transportation such as taxi supply-demand matching can be modeled using such collective planning models~\citep{varakantham2012decision,NguyenKL17}.

In this work, we focus on the \textit{collective} Dec-POMDP framework ({\cmdp}) that formalizes such a collective multiagent sequential decision making problem under uncertainty~\citep{NguyenKL17}.~\citeauthor{NguyenKL17} present a sampling based approach to optimize policies in the {\cmdp} model. 
A key drawback of this previous approach is that policies are represented in a tabular form which scales poorly with the size of observation space of agents. 
Motivated by the recent success of reinforcement learning (RL) approaches~\citep{Mnih15,schulman15,mniha16,FoersterAFW16,Leibo17}, our main contribution is a actor-critic (AC) reinforcement learning method~\citep{Konda:2003} for optimizing {\cmdp} policies. 
\begin{wrapfigure}{r}{6.6cm}
	\centering
	\includegraphics[scale=0.55]{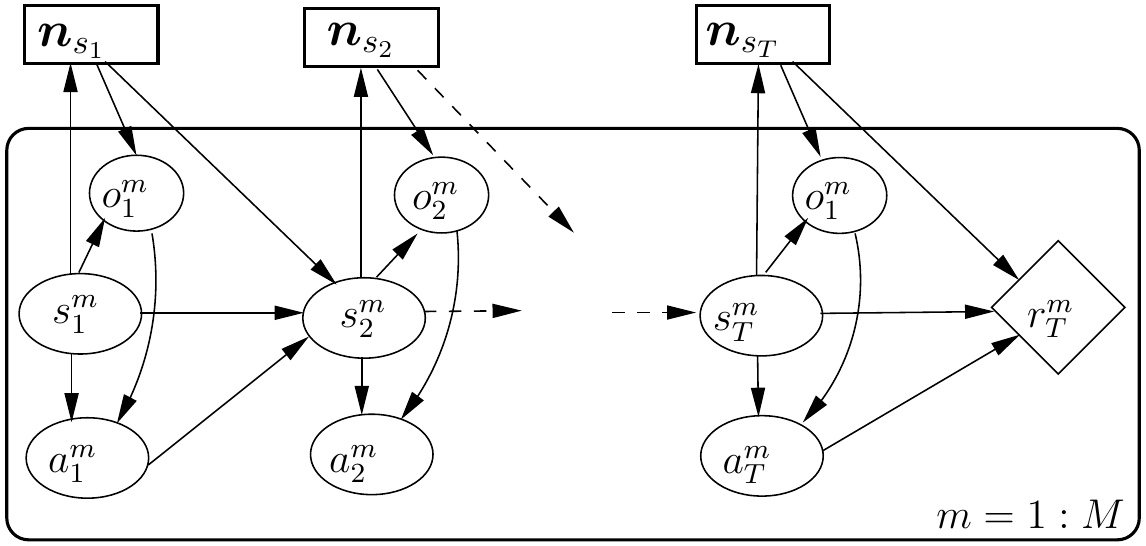}
\caption{\small T-step DBN for a {\cmdp}} 
	\label{fig:cmdp}
\end{wrapfigure}
Policies are represented using function approximator such as a neural network, thereby  avoiding the scalability issues of a tabular policy. 
We derive the policy gradient and develop a factored action-value approximator based on collective agent interactions in {\cmdp}s. 
Vanilla AC is slow to converge on large problems due to known issues of learning with \textit{global} reward in large multiagent systems~\citep{Bagnell:2005}. To address this, we also develop a new way to train the critic, our action-value approximator, that effectively utilizes local value function of agents. 

We test our approach on a synthetic multirobot grid navigation domain from~\citep{NguyenKL17}, and a real world supply-demand taxi matching problem in a large Asian city with up to 8000 taxis (or agents) showing the scalability of our approach to large multiagent systems. Empirically, our new factored actor-critic approach works better than previous best approaches providing much higher solution quality. The factored AC algorithm empirically converges much faster than the vanilla $\ac$ validating the effectiveness of our new training approach for the critic. 


\noindent{\textbf{Related work: }} Our work is based on the framework of policy gradient with approximate value function similar to \cite{Sutton:1999}. However, as we empirically show, directly applying the original policy gradient from \cite{Sutton:1999} into the multi-agent setting and specifically for the {\cmdp} model results in a high variance solution. In this work, we show a suitable form of compatible value function approximation for {\cmdp s} that results in an efficient and low variance policy gradient update. Reinforcement learning for decentralized policies has been studied earlier in \cite{peshkin2000learning}, \cite{aberdeen2006policy}. \cite{guestrin2002coordinated} also proposed using REINFORCE to train a softmax policy of a factored value function from the coordination graph. However in such previous works, policy gradient is estimated from the global empirical returns instead of a decomposed critic. We show in section \ref{sec:criticTraining} that having a decomposed critic along with an individual value function based training of this critic is important for sample-efficient learning. Our empirical results show that our proposed critic training has faster convergence than training with global empirical returns.

\section{Collective Decentralized POMDP Model}

We first describe the {\cmdp} model introduced in~\citep{NguyenKL17}. A $T$-step Dynamic Bayesian Network (DBN) for this model is shown using the plate notation in figure~\ref{fig:cmdp}. It consists of the following:
\squishlist
	\item A finite planning horizon $H$.
	\item The number of agents $M$. An agent $m$ can be in one of the states in the state space $S$. The joint state space is $\times_{m=1}^M S$. We denote a single state as $i\in S$.
	\item A set of action $A$ for each agent $m$. We denote an individual action as $j\in A$.
	\item Let $({s}_{1:H}, {a}_{1:H})^m\=(s_1^m,a_1^m,s_2^m \ldots, s_H^m, a_H^m)$ denote the complete state-action trajectory of an agent $m$. We denote the state and action of agent $m$ at time $t$ using random variables $s_t^m$, $a_t^m$. Different indicator functions $\mathbb{I}_{t}(\cdot)$ are defined in table~\ref{tab:not}. We define the following count given the trajectory of each agent $m\in M$:
	$$n_{t}(i, j, i') \hspace{0pt} \= \sum_{m =1}^M \mathbb{I}^m_t( i, j, i^\prime)\; \forall i, i' \!\!\in\!\! S, j \!\!\in\!\! A$$
As noted in table~\ref{tab:not}, count $n_{t}(i, j, i')$ denotes the number of agents in state $i$ taking action $j$ at time step $t$ and transitioning to next state $i'$; other counts, $n_t(i)$ and $n_t(i, j)$, are defined analogously. Using these counts, we can define the count tables $\bs{\n}_{\bs{s}_t}$ and $\bs{\n}_{\bs{s}_t\bs{a}_t}$ for the time step $t$ as shown in table~\ref{tab:not}.
	\item We assume a general partially observable setting wherein agents can have different observations based on the collective influence of other agents. An agent observes its local state $s_t^m$. In addition, it also observes $o_t^m$ at time $t$ based on its local state $s_t^m$ and the count table $\bs{\n}_{\bs{s}_t}$. E.g., an agent $m$ in state $i$ at time $t$ can observe the count of other agents also in state $i$ (=$n_{t}(i)$) or other agents in some neighborhood of the state $i$ (=$\{n_t(j)\; \forall j\in \text{Nb}(i)\}$).
	\item The transition function is $ \phi_t\big(s^m_{t+1}\= i'|s^m_{t}\=i, a^m_t\=j,  \bs{\n}_{\bs{s}_t} \big)$.  The transition function is the same for all the agents. Notice that it is affected by $\bs{\n}_{\bs{s}_t}$, which depends on the collective behavior of the agent population.
	\item Each agent $m$ has a non-stationary policy  $\pi^m_t(j| i,  o_t^m(i, \bs{\n}_{\bs{s}_t}))$ denoting the probability of agent $m$ to take action $j$ given its observation $(i,  o_t^m(i, \bs{\n}_{\bs{s}_t}))$ at time $t$. We denote the policy over the planning horizon of an agent $m$ to be $\pi^m = (\pi^m_1, \ldots, \pi^m_H)$.
	\item An agent $m$ receives the reward $r^m_t = r_t(i, j,  \bs{\n}_{\bs{s}_t})$ dependent on its local state and action, and  the counts $\bs{\n}_{\bs{s}_t}$.
	\item Initial state distribution, $b_o = (P(i)\forall i \in S)$, is  the same for all agents.
\squishend
We  present here the simplest version where all the agents are of the same type having similar state transition, observation and reward models. The model can handle multiple agent types where agents have different dynamics based on their type. We can also incorporate an \textit{external} state that is unaffected by agents' actions (such as taxi demand in transportation domain). Our results are extendible to address such settings also. 


\begin{table}[t]
\centering
\renewcommand{\arraystretch}{1.1}
\resizebox{5.5in}{!}{
\begin{tabular}{ l l }
  \hhline{==}
  \noalign{\smallskip}
  $\mathbb{I}^m_{t}(i) \!\in\! \{0, 1\}$ & if agent $m$ is at state $i$ at time $t$ or $s_t^m = i$ \\
  $\mathbb{I}^m_{t}(i, j)  \!\in\! \{0, 1\}$ & if agent $m$ takes action $j$ in state $i$  at time $t$ or $(s_t^m, a_t^m) = (i, j)$ \\
  $\mathbb{I}^m_{t}(i, j, i') \!\in\! \{0, 1\}$ & if agent $m$ takes action $j$ in state $i$ at  time $t$ and transitions to state $i^\prime$  or $(s_t^m, a_t^m, s_{t+1}^m) = (i, j, i^\prime)$  \\
  $n_t(i)\!\in\! [0;M]$& Number of agents at state $i$ at time $t$\\
  $n_t(i, j)\!\in\! [0;M]$ & Number of agents at state $i$ taking  action $j$ at time $t$\\
  $n_t(i, j, i^\prime)\!\in\! [0;M]$ & Number of agents at state $i$  taking action $j$ at time $t$ and transitioning to  state $i'$ at time $t+1$ \\
  $\bs{\n}_{\bs{s}_t}$ & Count table $(n_{t}(i)\; \forall i\!\in\! S)$ \\
  $\bs{\n}_{\bs{s}_t\bs{a}_t}$ &  Count table $(n_{t}(i, j)\; \forall i \!\in\! S, j\!\in\! A)$\\
  $\bs{\n}_{\bs{s}_t\bs{a}_t\bs{s}_{t+1}}$ &  Count table $(n_{t}(i, j, i')\; \forall i, i' \!\in\! S, j\!\in\! A)$\\
\end{tabular}}
\caption{\small Summary of notations given the state-action trajectories, $({s}_{1:H}, {a}_{1:H})^m\; \forall m$, for all the agents}
\label{tab:not}
\end{table}

Models such as {\cmdp}s are useful in settings where agent population is large, and agent identity does not affect the reward or the transition function. A motivating application of this model is for the taxi-fleet optimization where the problem is to compute policies for taxis such that the total profit of the fleet is maximized~\citep{varakantham2012decision,NguyenKL17}. The decision making for a taxi is as follows. At time $t$, each taxi observes its current city zone $z$ (different zones constitute the state-space $S$), and also the count of other taxis in the current zone and its neighboring zones as well as an estimate of the current local demand. This constitutes the count-based observation $o(\cdot)$ for the taxi. Based on this observation, the taxi must decide whether to stay in the current zone $z$ to look for passengers or \textit{move} to another zone. These decision choices depend on several factors such as the ratio of demand and the count of other taxis in the current zone. Similarly, the environment is stochastic with variable taxi demand at different times. Such historical demand data is often available using GPS traces of the taxi fleet~\citep{varakantham2012decision}. 


\noindent{\textbf{Count-Based statistic for planning: }}A key property in the {\cmdp} model is that the model dynamics depend on the collective interaction among agents rather than agent identities. In settings such as taxi fleet optimization, the agent population size can be quite large ($\approx 8000$ for our real world experiments). Given such a large  population, it is not possible to compute unique policy for each agent. Therefore, similar to previous work~\citep{varakantham2012decision,NguyenKL17}, our goal is to compute a homogenous policy $\pi$ for all the agents. As the policy $\pi$ is dependent on counts, it represents an expressive class of policies. 

For a fixed population $M$, let $\{({s}_{1:T}, {a}_{1:T})^m\; \forall m\}$ denote the state-action trajectories of different agents sampled from the DBN in figure~\ref{fig:cmdp}. Let $\bs{\n}_{1:T} \!\=\! \{(\bs{\n}_{\bs{s}_t}$, $\bs{\n}_{\bs{s}_t\bs{a}_t}, \bs{\n}_{\bs{s}_t\bs{a}_t\bs{s}_{t+1}})\; \forall t\=1\!:\!T\}$ be the combined vector of the resulting count tables for each time step $t$.~\citeauthor{NguyenKL17} show that counts $\bs{\n}$ are the \textit{sufficient statistic} for planning. That is, the joint-value function of a policy $\pi$ over horizon $H$ can be computed by the expectation over counts as~\citep{NguyenKL17}:
\begin{align}
\label{eq:vf}V(\pi) = \sum_{m=1}^M\sum_{T=1}^H \mathbb{E}[r^m_T] = \sum_{\bs{\n}\in \Omega_{1:H}} P(\bs{\n};\pi) \bigg[ \sum_{T=1}^H \sum_{i\in S, j\in A}n_{T}(i,j) r_T\big(i, j, \bs{\n}_{T}\big) \bigg]
\end{align}
\noindent Set $\Omega_{1:H}$ is the set of all allowed consistent count tables as:
{
\begin{align*}
\label{eq:norm}&\sum_{i\in S}n_{T}(i) \= M\; \forall T\;;\; \sum_{j\in A} n_{T}(i, j) \= n_{T}(i)\; \forall j \forall T\;; \sum_{i'\in S} n_{T}(i, j, i') \= n_{T}(i, j) \; \forall i\in S, j\in A, \forall T
\end{align*}}
$P(\bs{\n};\pi)$ is the distribution over counts (detailed expression in appendix). A key benefit of this result is that we can evaluate the policy $\pi$ by sampling counts $\bs{\n}$ directly from  $P(\bs{\n})$ without sampling individual agent trajectories $({s}_{1:H}, {a}_{1:H})^m$ for different agents, resulting in significant computational savings. Our goal is to compute the optimal policy $\pi$ that maximizes $V(\pi)$. We assume a RL setting with centralized learning and decentralized execution. We assume a simulator is available that can provide count samples from $P(\bs{\n};\pi)$.

\section{Policy Gradient for {\cmdp}s}

Previous work proposed an expectation-maximization (EM)~\citep{Dempster77} based sampling approach to optimize the policy $\pi$~\citep{NguyenKL17}. The policy is represented as a piecewise linear tabular policy over the space of counts $\bs{\n}$ where each linear piece specifies a distribution over next actions. However, this tabular representation is limited in its expressive power as the number of pieces is fixed apriori, and the range of each piece has to be defined manually which can adversely affect performance. Furthermore, exponentially many pieces are required when the observation $o$ is multidimensional (i.e., an agent observes counts from some local neighborhood of its location). To address such issues, our goal is to optimize policies in a functional form such as a neural network.


We first extend the policy gradient theorem of~\citep{Sutton:1999} to {\cmdp}s. Let $\theta$ denote the vector of policy parameters. We next show how to compute $\nabla_{\theta}V(\pi)$. 
Let $\bs{s}_t$, $\bs{a}_t$ denote the joint-state and joint-actions of all the agents at time $t$. The value function of a given policy $\pi$ in an expanded form is given as:
\begin{align}
V_t(\pi) = \sum_{\bs{s}_{t}, \bs{a}_{t}} P^{\pi}(\bs{s}_{t}, \bs{a}_{t}|b_o,\pi)Q^{\pi}_t(\bs{s}_{t}, \bs{a}_{t})
\end{align}
where $ P^{\pi}(\bs{s}_{t}, \bs{a}_{t}|b_o) = \sum_{\bs{s}_{1:t-1},\bs{a}_{1:t-1}} P^{\pi}(\bs{s}_{1:t}, \bs{a}_{1:t}|b_o) $ is the distribution of the joint state-action $ \bs{s}_{t}, \bs{a}_{t} $ under the policy $\pi$. The value function  $ Q^{\pi}_t(\bs{s}_{t}, \bs{a}_{t}) $ is computed as:
\begin{align}
\label{eq:qfunc}
Q^{\pi}_t(\bs{s}_{t}, \bs{a}_{t}) = r_t(\bs{s}_{t}, \bs{a}_{t}) + \sum_{\bs{s}_{t+1}, \bs{a}_{t+1}}P^{\pi}(\bs{s}_{t+1}, \bs{a}_{t+1}|\bs{s}_{t}, \bs{a}_{t})Q^{\pi}_{t+1}(\bs{s}_{t+1}, \bs{a}_{t+1})
\end{align}
We next state the policy gradient theorem for {\cmdp}s:
\begin{theorem}
For any {\cmdp}, the policy gradient is given as:
\begin{equation}
\label{eq:normalPolicyUpdate}
\nabla_\theta V_1(\pi)= \sum_{t=1}^H E_{\bs{s}_t,\bs{a}_t|b_o, \pi} \bigg[ Q^{\pi}_t(\bs{s}_t,\bs{a}_t) \sum_{i\in S, j\in A}{n_{t}}(i,j)  \nabla_\theta  \log \pi_t\big(j|i, o(i, \bs{\n}_{\bs{s}_t})\big)\bigg]
\end{equation} 
\end{theorem}
The proofs of this theorem and other subsequent results are provided in the appendix.

Notice that computing the policy gradient using the above result is not practical for multiple reasons. The space of join-state action $(\bs{s}_t,\bs{a}_t)$ is combinatorial. Given that the agent population size can be large, sampling each agent's trajectory is not computationally tractable. To remedy this, we later show how to compute the gradient by directly sampling counts $\bs{\n}\!\sim\! P(\bs{\n}; \pi)$ similar to policy evaluation in~\eqref{eq:vf}. Similarly, one can estimate the action-value function $Q^{\pi}_t(\bs{s}_t,\bs{a}_t)$ using empirical returns as an approximation. This would be the analogue of the standard REINFORCE algorithm~\citep{Williams1992} for {\cmdp}s. It is well known that REINFORCE may learn slowly than other methods that use a learned action-value function~\citep{Sutton:1999}. Therefore, we next present a function approximator for $Q^{\pi}_t$, and show the computation of policy gradient by directly sampling counts $\bs{\n}$.

\subsection{Policy Gradient with Action-Value Approximation}

One can approximate the action-value function $Q^{\pi}_t(\bs{s}_t,\bs{a}_t)$ in several different ways. We consider the following special form of the approximate value function $f_w$:
\begin{align}
\label{eq:facvf}
Q^{\pi}_t(\bs{s}_t,\bs{a}_t)\approx f_w(\bs{s}_t,\bs{a}_t) = \sum_{m=1}^M f^{m}_{w}\big(s^m_t, o(s^m_t, \bs{\n}_{\bs{s}_t}),a^m_t \big)
\end{align}
where each $f^{m}_{w}$ is defined for each agent $m$ and takes as input the agent's local state, action and the observation. Notice that different components $f_w^m$ are correlated as they depend on the common count table $\bs{\n}_{\bs{s}_t}$. Such a decomposable form is useful as it leads to efficient policy gradient computation. Furthermore,  an important class of approximate value function having this form for {\cmdp}s is the \textit{compatible value function}~\citep{Sutton:1999} which results in an unbiased policy gradient (details in appendix).


\begin{proposition}
	Compatible value function for {\cmdp}s can be factorized as:
	\begin{equation}
	f_w(\bs{s}_t,\bs{a}_t) = \sum_{m} f^{m}_{w}(s^m_t, o(s^m_t, \bn_{\bs{s}_t}),a^m)\nonumber
	\end{equation}	
\end{proposition}
We can directly replace $Q^\pi(\cdot)$ in policy gradient~\eqref{eq:normalPolicyUpdate} by the approximate action-value function $f_w$. Empirically, we found that variance using this estimator was high. We exploit the structure of $f_w$ and show further factorization of the policy gradient next which works much better empirically.
\begin{theorem}
	\label{thm:valueFunctionDecomposition}
	For any value function having the decomposition as:
	\begin{align}
	f_w(\bs{s}_t,\bs{a}_t) = \sum_{m} f^{m}_{w}\big(s^m_t, o(s^m_t, \bs{\n}_{\bs{s}_t}),a^m_t\big),
	\end{align} 
	the \textit{policy gradient} can be computed as
	\begin{align}
	\label{eq:polgrad}
	\nabla_\theta V_1(\pi) = \sum_{t=1}^H \mathbb{E}_{\bs{s}_t, \bs{a}_t}\Big[\sum_{m}\nabla_\theta  \log \pi\big(a^m_t|s_t^m, o(s_t^m, \bs{\n}_{\bs{s}_t})\big) f^{m}_{w}\big(s^m_t, o(s^m_t, \bs{\n}_{\bs{s}_t}),a^m_t\big)\Big]
	\end{align}
\end{theorem}
The above result shows that if the approximate value function is factored, then the resulting policy gradient also becomes factored. The above result also applies to agents with multiple types as we assumed the function $f_w^m$ is different for each agent. In the simpler case when all the agents are of same type, then we have the same function $f_w$ for each agent, and also deduce the following:
\begin{align}
\label{eq:approxvf}
f_w(\bs{s}_t,\bs{a}_t) = \sum_{i, j} n_t(i, j) f_w\big(i, j, o(i, \bs{\n}_{\bs{s}_t})\big)
\end{align}
Using the above result, we simplify the policy gradient as:
\begin{align}
\label{eq:gradapprox}\nabla_\theta V_1(\pi) = \sum_t\mathbb{E}_{\bs{s}_t, \bs{a}_t}\Big[\sum_{i,j} n_t(i,j) \nabla_\theta \log \pi\big(j|i, o(i, \bs{\n}_{\bs{s}_t})\big) f_{w}(i, j, o(i, \bs{\n}_{\bs{s}_t}))\Big]
\end{align}

\subsection{Count-based Policy Gradient Computation}

Notice that in~\eqref{eq:gradapprox}, the expectation is still w.r.t. joint-states and actions $(\bs{s}_t, \bs{a}_t)$ which is not efficient in large population sizes. To address this issue, we exploit the insight that the approximate value function in~\eqref{eq:approxvf} and the inner expression in~\eqref{eq:gradapprox} depends only on the counts generated by the joint-state and action $(\bs{s}_t,\bs{a}_t)$.
\begin{theorem}
For any value function having the form: $f_w(\bs{s}_t,\bs{a}_t) = \sum_{i, j} n_t(i, j) f_w\big(i, j, o(i, \bs{\n}_{\bs{s}_t})\big)$, the \textit{policy gradient} can be computed as:
	\begin{align}
	\label{eq:FpolicyUpdate}
	\mathbb{E}_{\bs{\n}_{1:H}\in \Omega_{1:H}} \bigg[ \sum_{t=1}^H\sum_{i\in S,j\in A} n_{t}(i,j)\nabla_\theta \log \pi\big(j|i, o(i, \bs{\n}_{t})\big) f_{w}(i, j, o(i, \bs{\n}_{t}))\bigg]
	\end{align}
\end{theorem}
The above result shows that the policy gradient can be computed by sampling count table vectors $\bs{\n}_{1:H}$ from the underlying distribution $P(\cdot)$ analogous to computing the value function of the policy  in~\eqref{eq:vf}, which is tractable even for large population sizes.

\section{Training Action-Value Function}
\label{sec:criticTraining}
In our approach, after count samples $\bs{\n}_{1:H}$ are generated to compute the  policy gradient, we also need to adjust the parameters $w$ of our critic $f_w$. Notice that as per~\eqref{eq:approxvf}, the action value function $f_w(\bs{s}_t,\bs{a}_t)$ depends only on the counts generated by the joint-state and action $(\bs{s}_t,\bs{a}_t)$. Training $f_w$ can be done by taking a gradient step to minimize the following loss function:
\begin{align}
\min_w \sum_{\xi=1}^K \sum_{t=1}^H \Big(f_w(\bs{\n}^\xi_{t}) - R^{\xi}_t\Big)^2 \label{eq:gloerror}
\end{align}
where $\bs{\n}^\xi_{1:H}$ is a count sample generated from the distribution $P(\bs{\n}; \pi)$; $f_w(\bs{\n}^\xi_{t})$ is the action value function and $R^{\xi}_t$ is the total empirical return for time step $t$ computed  using~\eqref{eq:vf}:
\begin{align}
&f_w(\bs{\n}^\xi_{t})\= \sum_{i,j} n_t^\xi(i, j) f_w(i, j, o(i, \bs{\n}^\xi_t));\;\; R^{\xi}_t = \sum_{T=t}^H \sum_{i\in S, j\in A} n_T^\xi(i, j) r_T(i, j,  \bs{\n}^\xi_T) \label{eq:glosig}
\end{align}
However, we found that the loss in~\eqref{eq:gloerror} did not work well for training the critic $f_w$ for larger problems. Several count samples were required to reliably train $f_w$ which adversely affects scalability for large problems with many agents. It is already known in multiagent RL that algorithms that solely rely on the \textit{global} reward signal (e.g. $R^{\xi}_t$ in our case) may require several more samples than approaches that take advantage of local reward signals~\citep{Bagnell:2005}. Motivated by this observation, we next develop a local reward signal based strategy to train the critic $f_w$.

\noindent\textbf{Individual Value Function: } Let $\bs{\n}^\xi_{1:H}$ be a count sample. Given the count sample $\bs{\n}^\xi_{1:H}$, let {\small$V_t^\xi(i, j)=\mathbb{E}[\sum_{t'=t}^H r_{t'}^m |s^m_t = i, a^t_m = j, n_{1:H}^\xi  ]$} denote the total expected reward obtained by an agent that is in state $i$ and takes action $j$ at time $t$. This \textit{individual} value function can be computed using dynamic programming as shown in~\citep{NguyenKL17}. Based on this value function, we next show an alternative reparameterization of the global empirical reward $R^{\xi}_t$ in~\eqref{eq:glosig}:
\begin{lemma}
\label{lem:local}
The empirical return $R_t^\xi$ for the time step $t$ given the count sample  $\bs{\n}^\xi_{1:H}$ can be re-parameterized as: $R^{\xi}_t  = \sum_{i\in S, j\in A} n_t^\xi(i, j) V_t^\xi(i, j)$.
\end{lemma}
\noindent\textbf{Individual Value Function Based Loss: } Given lemma~\ref{lem:local}, we next derive an upper bound on the on the true loss~\eqref{eq:gloerror} which effectively utilizes individual value functions:
\begin{align}
\sum_{\xi}\sum_t \Big(f_w(\bs{n}^\xi) - R^{\xi}_t\Big)^2&=\sum_{\xi}\sum_t\Big(\sum_{i,j} n^\xi_{t}(i,j) f_{w}(i, j, o(i, \bs{\n}^\xi_{t})) - \sum_{i,j} n^{\xi}_{t}(i,j) V_t^{\xi}(i,j)\Big)^2\nonumber\\
&=\sum_{\xi}\sum_t\bigg(\sum_{i, j}n^\xi_t(i, j)\Big(f_{w}(i, j, o(i, \bs{\n}^\xi_{t})) -  V_t^{\xi}(i,j)\Big) \bigg)^2\label{eq:normalValueFuncUpdate}\\
&\le M\sum_{\xi}\sum_{t,i,j}n_t(i, j)\Big(f_{w}(i, j, o(i, \bs{\n}^\xi_{t}))  - V_t^{\xi}(i,j)\Big)^2 \label{eq:upperbound}
\end{align}
where the last relation is derived by Cauchy-Schwarz inequality. 
We train the critic using the modified loss function in~\eqref{eq:upperbound}. Empirically, we observed that for larger problems, this new loss function in \eqref{eq:upperbound} resulted in much faster convergence than the original loss function in~\eqref{eq:normalValueFuncUpdate}.  Intuitively, this is because the new loss~\eqref{eq:upperbound} tries to adjust each critic component $f_{w}(i, j, o(i, \bs{\n}^\xi_{t}))$ closer to its counterpart empirical return $V_t^{\xi}(i,j)$. However, in the original loss function~\eqref{eq:normalValueFuncUpdate}, the focus is on minimizing the global loss, rather than adjusting each individual critic factor $f_{w}(\cdot)$ towards the corresponding empirical return.

\IncMargin{1.5em}
\begin{algorithm}[tp]
\small
\setstretch{1.0}
  \SetAlgoLined\DontPrintSemicolon
  \SetKwProg{myalg}{Algorithm}{}{}
  Initialize network parameter $\theta$ for actor $\pi$ and and $w$ for critic $f_w$\;
  $\alpha\leftarrow$ actor learning rate \;
  $\beta\leftarrow$ critic learning rate \;
\Repeat{\textbf{convergence}}
{
	Sample count vectors $\bs{\n}^{\xi}_{1:H} \sim P(\bs{n}_{1:H};\pi)\; \forall \xi= 1 \text{ to } K$\;
	Update critic as:\\
	$\fc_: w = w - \beta \frac{1}{K} \nabla_w\Big[\sum_{\xi}\sum_{t,i,j}n^{\xi}_t(i, j)\Big(f_{w}(i, j, o(i,  \bs{\n}^{\xi}_t))  - V_t^{\xi}(i,j)\Big)^2 \Big]$\;
	$\c_ \;: w = w - \beta \frac{1}{K} \nabla_w\Big[\sum_{\xi} \sum_t \Big(\sum_{i,j}n^{\xi}_t(i, j) f_{w}(i, j, o(i,  \bs{\n}^{\xi}_t))  - \sum_{i,j}n^{\xi}_t(i, j) V_t^{\xi}(i,j)\Big)^2 \Big]$\;
	Update actor as:\;
	$\fa_: \theta = \theta + \alpha \frac{1}{K} \nabla_{\theta} \sum_{\xi}\sum_t\Big[\sum_{i,j} n^{\xi}_{t}(i,j) \log \pi\big(j|i, o(i, \bs{\n}^{\xi}_{t})\big) f_{w}(i, j, o( \bs{\n}^{\xi}_t,i))\Big]$\;
	$\a_\; : \theta =\theta + \alpha \frac{1}{K} \nabla_{\theta} \sum_{\xi}\sum_t\Big[\sum_{i,j} n^{\xi}_{t}(i,j) \log \pi\big(j|i, o(i, \bs{\n}^{\xi}_{t})\big)\Big] \Big[\sum_{i,j} n^{\xi}_{t}(i,j) f_{w}(i, j, o( \bs{\n}^{\xi}_t,i))\Big]$\;
}
\Return $\theta, w$  \; 
\caption{\small Actor-Critic RL for {\cmdp}s}
\label{algo:rl}
\end{algorithm}
\DecMargin{1em}
Algorithm~\ref{algo:rl} shows the outline of our AC approach for {\cmdp}s. Lines 7 and 8 show two different options to train the critic. Line 7 represents critic update based on local value functions, also referred to as factored critic update ($\fc_$). Line 8 shows update based on global reward or global critic update ($\c_$). Line 10 shows the policy gradient computed using theorem~\ref{thm:valueFunctionDecomposition} ($\fa_$). Line 11 shows how the gradient is computed by directly using $f_w$ from eq.~\eqref{eq:facvf} in eq.~\ref{eq:normalPolicyUpdate}.


\section{Experiments}

This section compares the performance of our AC approach with two other approaches for solving {\cmdp}s---Soft-Max based flow update (SMFU)~\citep{varakantham2012decision}, and the Expectation-Maximization (EM) approach~\citep{NguyenKL17}. SMFU  can only optimize policies where an agent's action only depends on its local state, $\pi(a_t^m|s_t^m)$, as it approximates the effect of counts $\bs{\n}$ by computing the \textit{single} most likely count vector during the planning phase. The EM approach can optimize count-based piecewise linear policies where $\pi_t(a_t^m|s_t^m, \cdot)$ is a piecewise function over the space of all possible count observations $o_t$.

Algorithm 1 shows two ways of updating the critic (in lines 7, 8) and two ways of updating the actor (in lines 10, 11) leading to 4 possible settings for our actor-critic approach---$\fafc$, $\ac$, $\afc$, $\fac$. We also investigate the properties of these different actor-critic approaches. The neural network structure and other experimental settings are provided in the appendix.

For fair comparisons with previous approaches, we use three different models for counts-based observation $o_t$. In `o0' setting, policies depend only on agent's local state $s_t^m$ and not on counts. In `o1' setting, policies depend on the local state $s_t^m$ and the single count observation $n_t(s_t^m)$. That is, the agent can only observe the count of other agents in its current state $s_t^m$. In `oN' setting, the agent observes its local state $s_t^m$ and also the count of other agents from a local neighborhood (defined later) of the state $s_t^m$. The `oN' observation model provides the most information to an agent. However, it is also much more difficult to optimize as policies have more parameters. The SMFU only works with `o0' setting; EM and our actor-critic approach work for all the settings.

\textbf{Taxi Supply-Demand Matching: } We test our approach on this real-world domain described in section~2, and introduced in~\citep{varakantham2012decision}. In this problem, the goal is to compute taxi policies for optimizing the total revenue of the fleet. The data contains GPS traces of taxi movement in a large Asian city over 1 year. We use the observed demand information extracted from this dataset. 
On an average, there are around 8000 taxis per day (data is not exhaustive over all taxi operators). The city is divided into 81 zones and the plan horizon is 48 half hour intervals over 24 hours. For details about the environment dynamics, we refer to~\citep{varakantham2012decision}.

\begin{figure}[!t]
	\centering
	\begin{subfigure}[b]{0.49\textwidth}
		\centering
		\includegraphics[width=\textwidth]{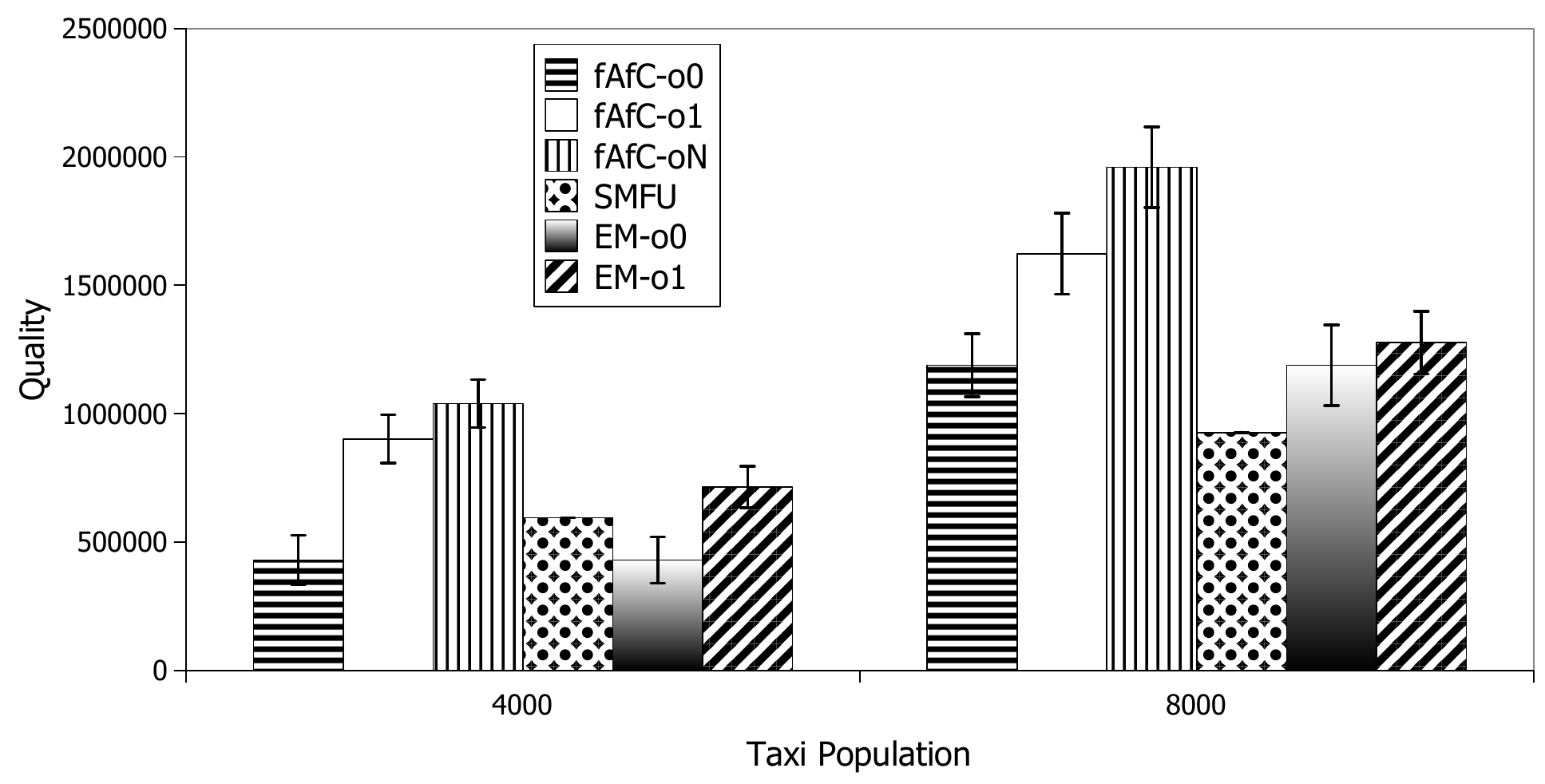}
		\caption{\small Solution quality with varying taxi population} 
	\end{subfigure}
	\begin{subfigure}[b]{0.49\textwidth}
		\centering
		\includegraphics[width=\textwidth]{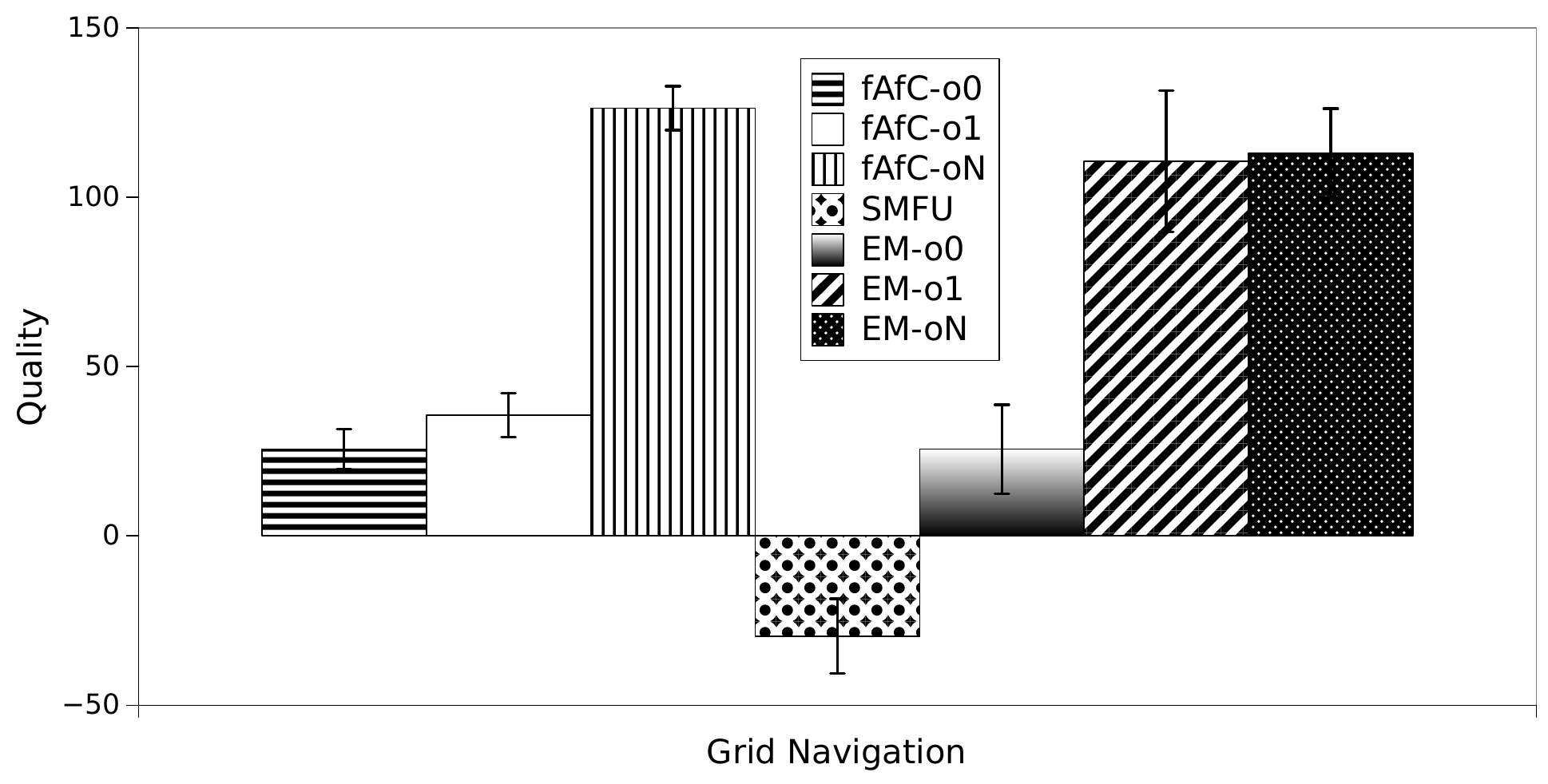}
		\caption{\small Solution quality in grid navigation problem} 
	\end{subfigure}
	\caption{\small Solution quality comparisons on the taxi problem and the grid navigation}
	\label{fig:results}
\end{figure}
\begin{figure}[!t]
	\centering	
	\begin{subfigure}[b]{\textwidth}
		\centering
		\includegraphics[scale=0.32]{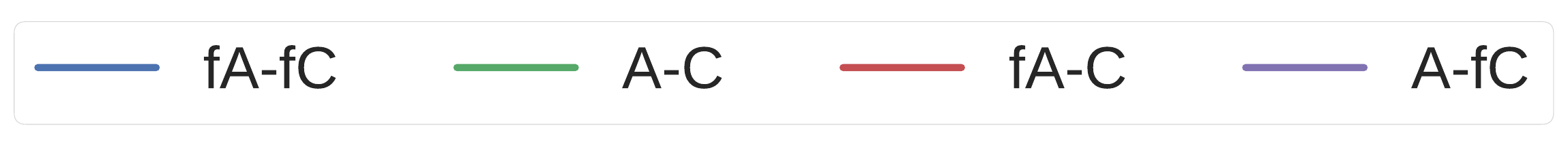}
	\end{subfigure}
	\\
	\begin{subfigure}[b]{0.32\textwidth}
		\centering
		\includegraphics[scale=0.23]{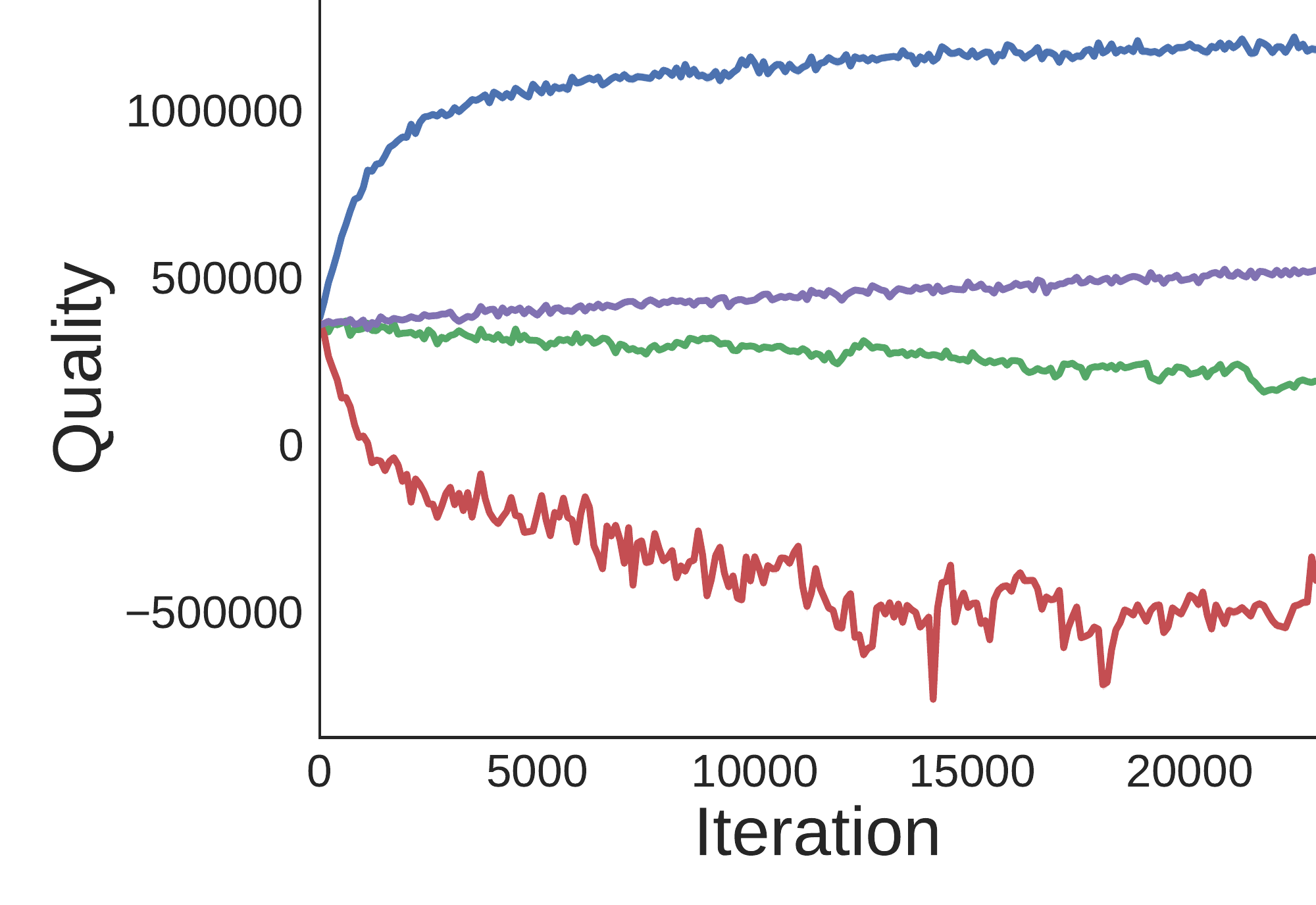}
		\caption{\small AC convergence with `o0'} 
	\end{subfigure}\hspace{1.7pt}
	\begin{subfigure}[b]{0.32\textwidth}
		\centering
		\includegraphics[scale=0.23]{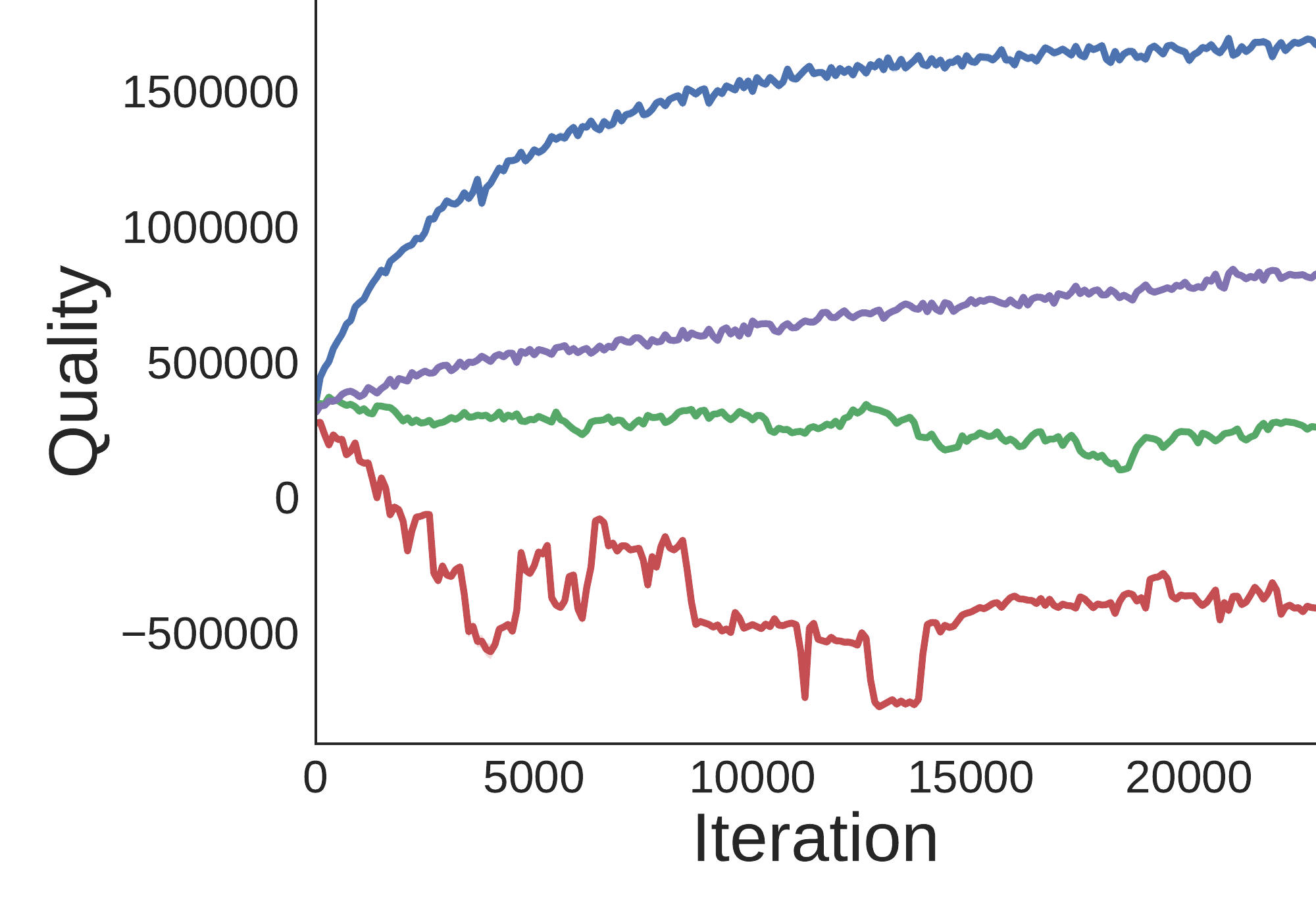}
		\caption{\small AC convergence with `o1'} 
	\end{subfigure}\hspace{1.7pt}
	\begin{subfigure}[b]{0.32\textwidth}
		\centering
		\includegraphics[scale=0.23]{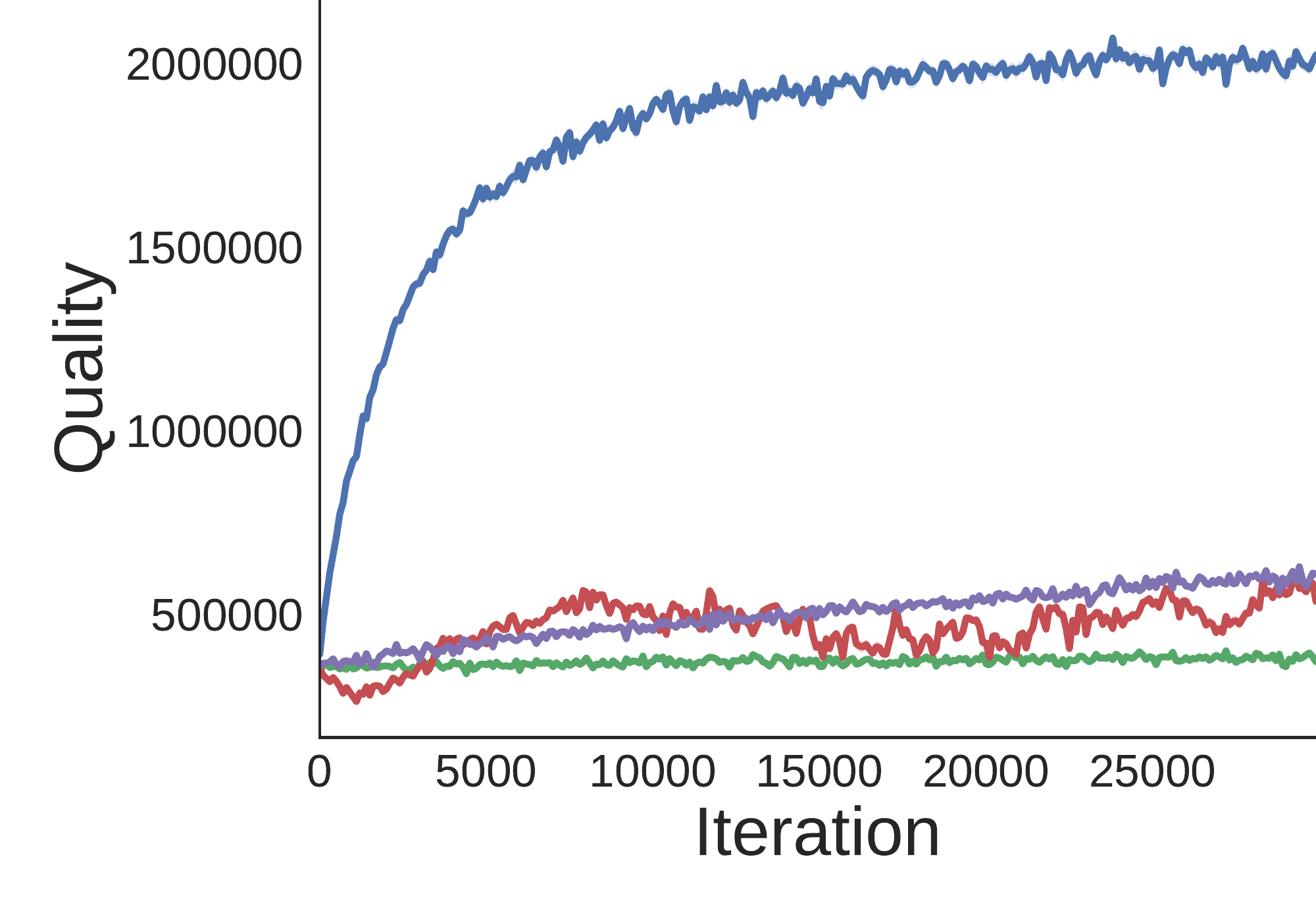}
		\caption{\small AC convergence with `oN'} 
	\end{subfigure}
	\caption{\small Convergence of different actor-critic variants on the taxi problem with 8000 taxis }
	\label{fig:conv}
\end{figure}

Figure~\ref{fig:results}(a) shows the quality comparisons among different approaches with different observation models (`o0', `o1' and `oN'). We test with total number of taxis as 4000 and 8000 to see if taxi population size affects the relative performance of different approaches. The y-axis shows the average per day profit for the entire fleet. For the `o0' case, all approaches ($\fafc$-`o0', SMFU, EM-`o0') give similar quality with $\fafc$-`o0' and  EM-`o0' performing slightly better than SMFU for the 8000 taxis. For the `o1' case, there is sharp improvement in quality by $\fafc$-`o1' over $\fafc$-`o0'  confirming that taking count based observation into account results in better policies. Our approach $\fafc$-`o1' is also significantly better than the policies optimized by EM-`o1' for both 4000 and 8000 taxi setting.

To further test the scalability and the ability to optimize complex policies by our approach in the `oN' setting, we define the neighborhood of each state (which is a zone in the city) to be the set of its geographically connected zones based on the zonal decomposition shown in~\citep{NguyenKL17}. On an average, there are about 8 neighboring zones for a given zone, resulting in 9 count based observations available to the agent for taking decisions. Each agent observes both the taxi count and the demand information from such neighboring zones. In figure~\ref{fig:results}(a), $\fafc$-`oN' result clearly shows that taking multiple observations into account significantly increases solution quality---$\fafc$-`oN' provides an increase of 64\% in quality over $\fafc$-`o0' and 20\% over $\fafc$-`o1' for the 8000 taxi case. For EM-`oN', we used a bare minimum of 2 pieces per observation dimension (resulting in $2^9$ pieces per time step). We observed that EM was unable to converge within 30K iterations and provided even worse quality than EM-`o1' at the end. These results show that despite the larger search space, our $\fafc$ approach can effectively optimize complex policies whereas the tabular policy based EM approach was ineffective for this case. 

Figures~\ref{fig:conv}(a-c) show the quality Vs. iterations for different variations of our actor critic approach---$\fafc$, $\ac$, $\afc$, $\fac$---for the `o0', `o1' and the `oN' observation model. These figures clearly show that using factored actor and the factored critic update in $\fafc$ is the most reliable strategy over all the other variations and for all the observation models. Variations such as $\ac$ and $\fac$ were not able to converge at all despite having exactly the same parameters as $\fafc$. These results validate different strategies that we have developed in our work to make vanilla AC converge faster for large problems.

\textbf{Robot navigation in a congested environment: } We also tested on a synthetic benchmark introduced in~\citep{NguyenKL17}. The goal is for a population of robots ($=20$) to move from a set of initial locations to a goal state in a 5x5 grid. If there is congestion on an edge, then each agent attempting to cross the edge has higher chance of action failure. Similarly, agents also receive a negative reward if there is edge congestion. On successfully reaching the goal state, agents receive a positive reward  and transition back to one of the initial state. We set the horizon to 100 steps.

Figure~\ref{fig:results}(b) shows the solution quality comparisons among different approaches. In the `oN' observation model, the agent observes its 4 immediate neighbor node's count information. In this problem, SMFU performed worst, $\fafc$ and EM both performed much better. As expected $\fafc$-`oN' provides the best solution quality over all the other approaches. In this domain, EM is competitive with $\fafc$ as for this relatively smaller problem with 25 agents, the space of counts is much smaller than in the taxi domain. Therefore, EM's piecewise policy is able to provide a fine grained approximation over the count range.

\section{Summary}

We addressed the problem of collective multiagent planning where the collective behavior of a population of agents affects the model dynamics. We developed a new actor-critic method for solving such collective planning problems within the {\cmdp} framework. We derived several new results for {\cmdp}s such as the policy gradient derivation, and the structure of the compatible value function. To overcome the slow convergence of the vanilla actor-critic method we developed multiple techniques based on value function factorization and training the critic using individual value function of agents. Using such techniques, our approach provided significantly better quality than previous approaches, and proved scalable and effective for optimizing policies in a real world taxi supply-demand problem and a synthetic grid navigation problem.
\section{ Acknowledgments}
This research project is supported by National Research Foundation Singapore under its Corp Lab @ University scheme and Fujitsu Limited. First author is also supported by A$^\star$STAR graduate scholarship.
\newpage

\bibliography{network}

\begin{thebibliography}{}

\bibitem[Aberdeen, 2006]{aberdeen2006policy}
Aberdeen, D. (2006).
\newblock Policy-gradient methods for planning.
\newblock In {\em Advances in Neural Information Processing Systems}, pages
  9--16.

\bibitem[Amato et~al., 2015]{Amato15}
Amato, C., Konidaris, G., Cruz, G., Maynor, C.~A., How, J.~P., and Kaelbling,
  L.~P. (2015).
\newblock Planning for decentralized control of multiple robots under
  uncertainty.
\newblock In {\em {IEEE} International Conference on Robotics and Automation,
  {ICRA}}, pages 1241--1248.

\bibitem[Bagnell and Ng, 2005]{Bagnell:2005}
Bagnell, J.~A. and Ng, A.~Y. (2005).
\newblock On local rewards and scaling distributed reinforcement learning.
\newblock In {\em International Conference on Neural Information Processing
  Systems}, pages 91--98.

\bibitem[Becker et~al., 2004a]{BZLaamas04}
Becker, R., Zilberstein, S., and Lesser, V. (2004a).
\newblock Decentralized {M}arkov decision processes with event-driven
  interactions.
\newblock In {\em Proceedings of the 3rd International Conference on Autonomous
  Agents and Multiagent Systems}, pages 302--309.

\bibitem[Becker et~al., 2004b]{Becker04JAIR}
Becker, R., Zilberstein, S., Lesser, V., and Goldman, C.~V. (2004b).
\newblock Solving transition independent decentralized {Markov} decision
  processes.
\newblock {\em Journal of Artificial Intelligence Research}, 22:423--455.

\bibitem[Bernstein et~al., 2002]{bern02}
Bernstein, D.~S., Givan, R., Immerman, N., and Zilberstein, S. (2002).
\newblock The complexity of decentralized control of {M}arkov decision
  processes.
\newblock {\em Mathematics of Operations Research}, 27:819--840.

\bibitem[Dempster et~al., 1977]{Dempster77}
Dempster, A.~P., Laird, N.~M., and Rubin, D.~B. (1977).
\newblock Maximum likelihood from incomplete data via the {EM} algorithm.
\newblock {\em Journal of the Royal Statistical society, Series B},
  39(1):1--38.

\bibitem[Foerster et~al., 2016]{FoersterAFW16}
Foerster, J.~N., Assael, Y.~M., de~Freitas, N., and Whiteson, S. (2016).
\newblock Learning to communicate with deep multi-agent reinforcement learning.
\newblock In {\em Advances in Neural Information Processing Systems}, pages
  2137--2145.

\bibitem[Guestrin et~al., 2002]{guestrin2002coordinated}
Guestrin, C., Lagoudakis, M., and Parr, R. (2002).
\newblock Coordinated reinforcement learning.
\newblock In {\em ICML}, volume~2, pages 227--234.

\bibitem[Ioffe and Szegedy, 2015]{ioffe2015batch}
Ioffe, S. and Szegedy, C. (2015).
\newblock Batch normalization: Accelerating deep network training by reducing
  internal covariate shift.
\newblock {\em arXiv preprint arXiv:1502.03167}.

\bibitem[Konda and Tsitsiklis, 2003]{Konda:2003}
Konda, V.~R. and Tsitsiklis, J.~N. (2003).
\newblock On actor-critic algorithms.
\newblock {\em SIAM Journal on Control and Optimization}, 42(4):1143--1166.

\bibitem[Kumar et~al., 2011]{Kumar2011}
Kumar, A., Zilberstein, S., and Toussaint, M. (2011).
\newblock Scalable multiagent planning using probabilistic inference.
\newblock In {\em Proceedings of the Twenty-Second International Joint
  Conference on Artificial Intelligence}, pages 2140--2146, Barcelona, Spain.

\bibitem[Kumar et~al., 2015]{kumar2015probabilistic}
Kumar, A., Zilberstein, S., and Toussaint, M. (2015).
\newblock Probabilistic inference techniques for scalable multiagent decision
  making.
\newblock {\em Journal of Artificial Intelligence Research}, 53(1):223--270.

\bibitem[Leibo et~al., 2017]{Leibo17}
Leibo, J.~Z., Zambaldi, V.~F., Lanctot, M., Marecki, J., and Graepel, T.
  (2017).
\newblock Multi-agent reinforcement learning in sequential social dilemmas.
\newblock In {\em International Conference on Autonomous Agents and Multiagent
  Systems}.

\bibitem[Meyers and Schulz, 2012]{meyers12}
Meyers, C.~A. and Schulz, A.~S. (2012).
\newblock The complexity of congestion games.
\newblock {\em Networks}, 59:252--260.

\bibitem[Mnih et~al., 2016]{mniha16}
Mnih, V., Badia, A.~P., Mirza, M., Graves, A., Lillicrap, T., Harley, T.,
  Silver, D., and Kavukcuoglu, K. (2016).
\newblock Asynchronous methods for deep reinforcement learning.
\newblock In {\em International Conference on Machine Learning}, pages
  1928--1937.

\bibitem[Mnih et~al., 2015]{Mnih15}
Mnih, V., Kavukcuoglu, K., Silver, D., Rusu, A.~A., Veness, J., Bellemare,
  M.~G., Graves, A., Riedmiller, M.~A., Fidjeland, A., Ostrovski, G., Petersen,
  S., Beattie, C., Sadik, A., Antonoglou, I., King, H., Kumaran, D., Wierstra,
  D., Legg, S., and Hassabis, D. (2015).
\newblock Human-level control through deep reinforcement learning.
\newblock {\em Nature}, 518(7540):529--533.

\bibitem[Nair et~al., 2005]{nair05}
Nair, R., Varakantham, P., Tambe, M., and Yokoo, M. (2005).
\newblock Networked distributed {POMDP}s: A synthesis of distributed constraint
  optimization and {POMDP}s.
\newblock In {\em AAAI Conference on Artificial Intelligence}, pages 133--139.

\bibitem[Nguyen et~al., 2017]{NguyenKL17}
Nguyen, D.~T., Kumar, A., and Lau, H.~C. (2017).
\newblock Collective multiagent sequential decision making under uncertainty.
\newblock In {\em {AAAI} Conference on Artificial Intelligence}, pages
  3036--3043.

\bibitem[Pajarinen et~al., 2014]{PajarinenHP14}
Pajarinen, J., Hottinen, A., and Peltonen, J. (2014).
\newblock Optimizing spatial and temporal reuse in wireless networks by
  decentralized partially observable {Markov} decision processes.
\newblock {\em {IEEE} Trans. on Mobile Computing}, 13(4):866--879.

\bibitem[Peshkin et~al., 2000]{peshkin2000learning}
Peshkin, L., Kim, K.-E., Meuleau, N., and Kaelbling, L.~P. (2000).
\newblock Learning to cooperate via policy search.
\newblock In {\em Proceedings of the Sixteenth conference on Uncertainty in
  artificial intelligence}, pages 489--496. Morgan Kaufmann Publishers Inc.

\bibitem[Robbel et~al., 2016]{robbel2016exploiting}
Robbel, P., Oliehoek, F.~A., and Kochenderfer, M.~J. (2016).
\newblock Exploiting anonymity in approximate linear programming: Scaling to
  large multiagent {MDPs}.
\newblock In {\em AAAI Conference on Artificial Intelligence}, pages
  2537--2543.

\bibitem[Schulman et~al., 2015]{schulman15}
Schulman, J., Levine, S., Abbeel, P., Jordan, M., and Moritz, P. (2015).
\newblock Trust region policy optimization.
\newblock In {\em International Conference on Machine Learning}, pages
  1889--1897.

\bibitem[Sonu et~al., 2015]{sonu2015individual}
Sonu, E., Chen, Y., and Doshi, P. (2015).
\newblock Individual planning in agent populations: Exploiting anonymity and
  frame-action hypergraphs.
\newblock In {\em International Conference on Automated Planning and
  Scheduling}, pages 202--210.

\bibitem[Sutton et~al., 1999]{Sutton:1999}
Sutton, R.~S., McAllester, D., Singh, S., and Mansour, Y. (1999).
\newblock Policy gradient methods for reinforcement learning with function
  approximation.
\newblock In {\em International Conference on Neural Information Processing
  Systems}, pages 1057--1063.

\bibitem[van Hasselt et~al., 2016]{van2016learning}
van Hasselt, H., Guez, A., Hessel, M., Mnih, V., and Silver, D. (2016).
\newblock Learning values across many orders of magnitude.
\newblock {\em arXiv preprint arXiv:1602.07714}.

\bibitem[Varakantham et~al., 2014]{varakantham2014decentralized}
Varakantham, P., Adulyasak, Y., and Jaillet, P. (2014).
\newblock Decentralized stochastic planning with anonymity in interactions.
\newblock In {\em {AAAI} Conference on Artificial Intelligence}, pages
  2505--2511.

\bibitem[Varakantham et~al., 2012]{varakantham2012decision}
Varakantham, P.~R., Cheng, S.-F., Gordon, G., and Ahmed, A. (2012).
\newblock Decision support for agent populations in uncertain and congested
  environments.
\newblock In {\em {AAAI} Conference on Artificial Intelligence}, pages
  1471--1477.

\bibitem[Williams, 1992]{Williams1992}
Williams, R.~J. (1992).
\newblock Simple statistical gradient-following algorithms for connectionist
  reinforcement learning.
\newblock {\em Machine Learning}, 8(3):229--256.

\bibitem[Winstein and Balakrishnan, 2013]{Winstein:2013}
Winstein, K. and Balakrishnan, H. (2013).
\newblock Tcp ex machina: Computer-generated congestion control.
\newblock In {\em Proceedings of the ACM SIGCOMM 2013 Conference}, SIGCOMM '13,
  pages 123--134.

\bibitem[Witwicki and Durfee, 2010]{Witwicki10}
Witwicki, S.~J. and Durfee, E.~H. (2010).
\newblock Influence-based policy abstraction for weakly-coupled {Dec-POMDP}s.
\newblock In {\em International Conference on Automated Planning and
  Scheduling}, pages 185--192.

\end{thebibliography}
\newpage
\appendix
\appendixpage

\section{Distribution Over Counts}
\label{app:a}

We show the distribution directly over the count tables $\bs{\n}_{1:T}$ as shown in~\cite{NguyenKL17}. The distribution $P(\bn_{1:T};\pi)$ is defined as:
\begin{equation}
P(\bs{\n}_{1:T};\pi) = h(\bs{\n}_{1:T}) f(\bs{\n}_{1:T};\pi)
\label{eq:CDGM}
\end{equation}
where $f(\bs{\n}_{1:T};\pi)$ is given as:
\begin{align}
&f(\bs{\n}_{1:T};\pi) = \prod_{i\in S} P(i)^{n_{1}(i) } \prod_{t=1}^{T-1}\prod_{i,j,i'}  \bigg[\pi_t(j|i,  o_{t}(i, \bs{\n}_t^s)^{n_{t}(i,j)} \phi_t(i'|i, j, \bs{\n}_t^s)^{n_{t}(i, j, i')}\bigg]\nonumber\\
& \hspace{50pt} \prod_{i,j}  \pi_T(j | i, o_{T}(i, \bs{\n}_T^s))^{n_{T}(i, j)} \label{eq:joint}
\end{align}
where $\bs{\n}_t^s$ is the count table $(n_t(i)\; \forall i\in S)$ consisting of the count value for each state $i$ at time $t$.

The function $h(\bs{\n}_{1:T})$ counts the total number of ordered $M$ state-action trajectories with sufficient statistic equal to $\bs{\n}$, given as:
{
	\begin{align}
	&h(\bs{n}_{1:T})\= \frac{M!}{\prod_{i\in S}n_{1}(i)! } \bigg[  \prod_{t = 1}^{T-1}\prod_{i\in S}\frac{n_{t}(i)!}{\prod_{ i'\in S, j\in A}n_{t}(i, j, i')!}\bigg] \times\bigg[\prod_{i\in S}\frac{n_{t}(i)!}{\prod_{ j\in A}n_{t}(i, j)!}\bigg] \times \mathbb{I}[\bn_{1:T}\in \Omega_{1:T}] \label{eq:perm}
	\end{align}}

\noindent Set $\Omega_{1:T}$ is the set of all allowed consistent count tables as:
{
	\begin{align}
	\label{eq:norm}&\sum_{i\in S}n_{t}(i) \= M\; \forall t\;; \sum_{j\in A} n_{t}(i, j) \= n_{t}(i)\; \forall j, \forall t\\
	\label{eq:marg}&\sum_{i'} n_{t}(i, j, i') \= n_{t}(i, j) \; \forall i\in S, j\in A, \forall t
	\end{align}

	\section{Policy gradient in {\cmdp}s}
	\label{app:b}
	
	In following part, we show the policy gradient in {\cmdp}s with respect to the accumulated reward at the first time period $V_0$. The proof is similar to \cite{Sutton:1999}'s proof.
	\begin{align}
	&\frac{\partial V_0}{\partial \theta} = \sum_{\bs{s}_0,\bs{a}_0}\nabla_{\theta}\bigg( P^\pi(\bs{s}_0,\bs{a}_0|b_0, \pi)Q^{\pi}_0(\bs{s}_0,\bs{a}_0)\bigg)\\
	&= \sum_{\bs{s}_0,\bs{a}_0}Q^{\pi}_0(\bs{s}_0,\bs{a}_0)\nabla_\theta P^\pi(\bs{s}_0,\bs{a}_0|b_0, \pi)+  \sum_{\bs{s}_0,\bs{a}_0}P(\bs{s}_0,\bs{a}_0|b_0, \pi)\nabla_\theta Q^{\pi}_0(\bs{s}_0,\bs{a}_0) \\
	&\text{using the Q function definition for {\cmdp}s  and taking the derivative we get}\nonumber\\
	&= \sum_{\bs{s}_0,\bs{a}_0} Q^{\pi}_0(\bs{s}_0,\bs{a}_0)\nabla_\theta P^\pi(\bs{s}_0,\bs{a}_0|b_0, \pi)\nonumber\\
	&\hspace{0cm}+   \sum_{\bs{s}_0,\bs{a}_0}P(\bs{s}_0,\bs{a}_0|b_0, \pi)\nabla_\theta \Big[\sum_{\bs{s}_1,\bs{a}_1} P(\bs{s}_1,\bs{a}_1|\bs{s}_0,\bs{a}_0,\pi)Q^{\pi}_1(\bs{s}_1,\bs{a}_1)\Big] \\
	&\text{If we continue unrolling out the terms in the above expression, we get}\nonumber\\
	&= \sum_t\sum_{\bs{s}_{1:t},\bs{a}_{1:t}}Q^{\pi}_t(\bs{s}_t,\bs{a}_t)P(\bs{s}_{t-1},\bs{a}_{t-1}|b_0, \pi)\nabla_\theta P(\bs{s}_t,\bs{a}_t|\bs{s}_{t-1},\bs{a}_{t-1}; \pi) \label{eq:logTrick1}\\\
	&\text{this can be re-written use the log trick}\nonumber\\
	&= \sum_t\sum_{\bs{s}_{1:t},\bs{a}_{1:t}}Q^{\pi}_t(\bs{s}_t,\bs{a}_t)P(\bs{s}_t,\bs{a}_t|b_0, \pi)\nabla_\theta \log P(\bs{s}_t,\bs{a}_t|\bs{s}_{t-1},\bs{a}_{t-1}, \pi) \label{eq:logTrick2}\\
	&= E_{\bs{s}_t,\bs{a}_t|b_0, \pi}\Big[\sum_tQ^{\pi}_t(\bs{s}_t,\bs{a}_t)\nabla_\theta \log P(\bs{s}_t,\bs{a}_t|\bs{s}_{t-1},\bs{a}_{t-1}, \pi)\Big]\nonumber\\
	&= E_{\bs{s}_t,\bs{a}_t|b_0, \pi}\Big[\sum_tQ^{\pi}_t(\bs{s}_t,\bs{a}_t)\nabla_\theta \log P(\bs{a}_t|\bs{s}_t, \pi)\Big]\label{eq:CGMPolicyGradExact}
	\end{align}
	
	Next, we simplify the gradient term $\nabla_\theta \log P(\bs{a}_t|\bs{s}_t, \pi)$ as:
	\begin{proposition} We have
		\label{thm:policy-log-derivative}
		\begin{align}
		\nabla_\theta \log P(\bs{a}_t|\bs{s}_t ) =  \sum_m \nabla_\theta \log \Big( \pi_t^m(a^m_t|o(s^m_t, \bn_{\bs{s}_t})) \Big)
		\end{align}
	\end{proposition}
	\begin{proof}
		We simplify the above gradient as following: 
		\begin{align}
		\nabla_\theta \log P(\bs{a}_t|\bs{s}_t )  &= \nabla_\theta \log \Big( \prod_{m}\pi_t^m(a^m_t|o(s^m_t, \bn_{\bs{s}_t})) \Big) \nonumber\\
		&= \sum_m\nabla_\theta \log \Big( \pi_t^m(a^m_t|o(s^m_t, \bn_{\bs{s}_t})) \Big)\label{eq:MAPolicyGradExact}
		\end{align}
	\end{proof}
	Notice that we have proved the result in a general setting where each agent $m$ has a different policy $\pi^m$. In a homogeneous agent system (when each agent is of the same type and has the same policy $\pi$), the last equation  can be simplified by grouping agents taking similar action in similar state to give us:
	\begin{align}
	\label{eq:simplifiedLogPolicy}
	\nabla_\theta \log P(\bs{a}_t|\bs{s}_t ) = \sum_{i\in S, j\in A}{n_{t}}(i,j)  \nabla_\theta  \log \pi_t(j|o(i, \bn_{\bs{s}_t}))
	\end{align}
	Using the above results, the final policy gradient expression for {\cmdp}s is readily proved.
	\begin{theorem}
		For any {\cmdp}, the policy gradient is given as:
		\begin{equation}
		\label{eq:normalPolicyUpdate}
		\nabla_\theta V_1(\pi)= \sum_{t=1}^H E_{\bs{s}_t,\bs{a}_t|b_o, \pi} \bigg[ Q^{\pi}_t(\bs{s}_t,\bs{a}_t) \sum_{i\in S, j\in A}{n_{t}}(i,j)  \nabla_\theta  \log \pi_t\big(j|o(i, \bs{\n}_{\bs{s}_t})\big)\bigg]
		\end{equation} 
	\end{theorem}
	\begin{proof}
		This result is directly implied by substitute \eqref{eq:simplifiedLogPolicy} into \eqref{eq:CGMPolicyGradExact}.
	\end{proof}

	\section{Action Value Function Approximation For \cmdp}
	\label{sec:decomposedVF}
	
	We consider a special form of approximate value function
	\begin{equation}
	\label{eq:decomposedCountValueFunction}
	f_w(\bs{s}_t,\bs{a}_t) = \sum_{m} f^{m}_{w}(s^m_t, o(s^m_t, \bn_{\bs{s}_t}),a^m_t)
	\end{equation}
	
	There are 2 reasons to consider this form of approximate value function:
	\begin{itemize}
		\item This form will leads to the efficient update of policy gradient
		\item We can train this form efficiently if we can decompose the value function into sum of some individual value. Each component $ f^{m}_{w}(s^m_t, o(s^m_t, \bn_{\bs{s}_t}),a^m_t) $ can be understand as the contribution of each individual $m$ into the total value function.
	\end{itemize}
	
	One of important class of approximate value functions having this form is the \textit{compatible} value function. As shown in~\cite{Sutton:1999}, for compatible value functions, the policy gradient using the function approximator is equal to the true policy gradient. 
	
	\begin{proposition}
		The compatible value function approximation in {\cmdp}s has the form 
		\begin{equation}
		f_w(\bs{s}_t,\bs{a}_t) = \sum_{m} f^{m}_{w}(s^m_t, o(s^m_t, \bn_{\bs{s}_t}),a^m)\nonumber
		\end{equation}	
	\end{proposition}
	\begin{proof}
		Recall from \cite{Sutton:1999}, the compatible value function approximates the value function $ Q(\bs{s}_t, \bs{a}_t) $ with linear value $ f_w(\bs{s}_t, \bs{a}_t) = w^T \phi(\bs{s}_t, \bs{a}_t)$, where $w$ denotes function parameter vector and $ \phi(\bs{s}_t, \bs{a}_t) $ is compatible feature vector computed from the policy $\pi$ as
		\begin{equation}
		\phi(\bs{s}_t, \bs{a}_t) = \nabla_{\theta} \log P(\bs{a}_t|\bs{s}_t)
		\end{equation}
		Applying this for {\cmdp}s and using the result from proposition~\ref{thm:policy-log-derivative}, we have the linear compatible feature in a {\cmdp} to be:
		\begin{equation}
		\phi(\bs{s}_t, \bs{a}_t) = \nabla_{\theta} \log P(\bs{a}_t |\bs{s}_t) = \sum_m\nabla_{\theta} \log \pi_t^m(a^m|o(s^m_t, \bn_{\bs{s}_t}))
		\end{equation}
		We can rearrange $f_w(\bs{s}_t,\bs{a}_t)$ as follows
		\begin{align}
		f_w(\bs{s}_t,\bs{a}_t) &= w^T\phi(\bs{s}_t,\bs{a}_t) = w^T\Big[ \sum_m \nabla_{\theta} \log \pi_t(a^m|o(s^m_t, \bn_{\bs{s}_t}))\Big]\\
		&=  \sum_m w^T \nabla_{\theta} \log \pi_t(a^m|o(s^m_t, \bn_{\bs{s}_t}))
		\end{align}
		If we set $ f^{m}_{w}(s^m_t, o(s^m_t, \bn_{\bs{s}_t}),a^m) \= w^T \nabla_{\theta} \log \pi_t(a^m|o(s^m_t, \bn_{\bs{s}_t})) $, the theorem is proved.
	\end{proof}
	
	We also prove the next result in a general setting with each agent having a different policy $\pi^m$.
	\begin{theorem}
		\label{thm:valueFunctionDecomposition}
		For any value function having the decomposition as:
		\begin{align}
		f_w(\bs{s}_t,\bs{a}_t) = \sum_{m} f^{m}_{w}\big(s^m_t, o(s^m_t, \bs{\n}_{\bs{s}_t}),a^m_t\big),
		\end{align} 
		the \textit{policy gradient} can be computed as
		\begin{align}
		\label{eq:polgrad}
		\nabla_\theta V_1(\pi) = \sum_{t=1}^H \mathbb{E}_{\bs{s}_t, \bs{a}_t}\Big[\sum_{m}\nabla_\theta  \log \pi^m\big(a^m_t|s_t^m, o(s_t^m, \bs{\n}_{\bs{s}_t})\big) f^{m}_{w}\big(s^m_t, o(s^m_t, \bs{\n}_{\bs{s}_t}),a^m_t\big)\Big]
		\end{align}
	\end{theorem}
	\begin{proof}
		Substitute the approximate value function $f_w(\bs{s}_t,\bs{a}_t)$ to $Q^{\pi}(\bs{s}_t,\bs{a}_t)$ in the policy gradient formula~\eqref{eq:CGMPolicyGradExact}, we have the policy gradient computed by approximate value function $f_w(\bs{s}_t,\bs{a}_t)$ to be
		\begin{align}
		\nabla_\theta V_1(\pi)&=	\sum_t\mathbb{E}_{\bs{s}_t, \bs{a}_t}\Big[\nabla_\theta  \log P(\bs{a}_t |\bs{s}_t, \theta) f_w(\bs{s}_t,\bs{a}_t)\Big]\\
		&=	\sum_t\mathbb{E}_{\bs{s}_t, \bs{a}_t}\Big[\frac{\partial  \log \prod_m\pi^m\big(a^m_t|s_t^m, o(s_t^m, \bs{\n}_{\bs{s}_t})\big)}{\partial \theta}\big(\sum_{m'} f^{m'}_{w}(s^{m'}_t, o(s^{m'}_t,  \bs{\n}_{\bs{s}_t}),a^{m'}_t)\big)\Big]\\
		&=	\sum_t\mathbb{E}_{\bs{s}_t, \bs{a}_t}\Big[\sum_{m}\nabla_\theta  \log \pi^m\big(a^m_t|s_t^m, o(s_t^m, \bs{\n}_{\bs{s}_t})\big) \big(\sum_{m'} f^{m'}_{w}(s^{m'}_t, o(s^{m'}_t,  \bs{\n}_{\bs{s}_t}),a^{m'}_t)\big)\Big]\label{eq:step3}
		\end{align}
		Let us simplify the inner summation for a specific $ m, t $ by looking at:
		\begin{align}
		&\mathbb{E}_{\bs{s}_t, \bs{a}_t}\Big[\nabla_\theta  \log \pi^m\big(a^m_t|s_t^m, o(s_t^m, \bs{\n}_{\bs{s}_t})\big) \big(\sum_{m'\neq m} f^{m'}_{w}(s^{m'}_t, o(s^{m'}_t,  \bs{\n}_{\bs{s}_t}),a^{m'}_t)\big)\Big]
		\end{align}
		Given the independence of value functions of other agents $m'\neq m$ w.r.t. the action $a_t^m$ of agent $m$, we have:
		\begin{align}
		&=\mathbb{E}_{\bs{s}_t}\bigg[\mathbb{E}_{a^m_t|\bs{s}_t} \bigg(\nabla_\theta  \log \pi^m\big(a^m_t|s_t^m, o(s_t^m, \bs{\n}_{\bs{s}_t}) \big)\sum_{m'\neq m} \mathbb{E}_{a^{m'}_t|\bs{s}_t} f^{m'}_{w}(s^{m'}_t, o(s^{m'}_t, \bn_t),a^{m'}_t)\bigg)\bigg]\\
		&=\mathbb{E}_{\bs{s}_t}\bigg[ \mathbb{E}_{a^m_t|\bs{s}_t} \bigg(\nabla_\theta  \log \pi^m\big(a^m_t|s_t^m, o(s_t^m, \bs{\n}_{\bs{s}_t}) \times constant\_to\_a^m_t\bigg) \bigg]\\
		&=0
		\end{align}
		Applying this to \eqref{eq:step3}, we can dismiss all the term of $ m'\neq m $ to simplify \eqref{eq:step3} into \eqref{eq:polgrad}.
	\end{proof}
	
	In the setting all agents are identical with same policy $\pi$, we can use the following simplification of the approximate action-value function:
	\begin{equation}
	\label{eq:DecomposedCGMVF}
	f_w(\bs{s}_t,\bs{a}_t) =  \sum_{m} f_{w}(s^m_t, o(s^m_t, \bn_{\bs{s}_t}),a^m_t) = \sum_{i, j}n_{t}(i,j)f_{w}(i, j, o(i, \bs{n}_{\bs{s}_t}) = f_w(\bs{n}_{\bs{s}_t\bs{a}_t})
	\end{equation} 
	
	\subsection{Count-based Policy Gradient Computation in {\cmdp}s}
	\begin{theorem}
		For any value function having the form: $$f_w(\bs{s}_t,\bs{a}_t) = \sum_{i, j} n_t(i, j) f_w\big(i, j, o(i, \bs{\n}_{\bs{s}_t})\big),$$ the \textit{approximate policy gradient} can be computed as:
		\begin{align}
		\label{eq:FpolicyUpdate}
		\mathbb{E}_{\bs{\n}_{1:H}\in \Omega_{1:H}} \bigg[ \sum_{t=1}^H\sum_{i\in S,j\in A} n_{t}(i,j)\nabla_\theta \log \pi\big(j|i, o(i, \bs{\n}_{t})\big) f_{w}(i, j, o(i, \bs{\n}_{t}))\bigg]
		\end{align}
	\end{theorem}
	\begin{proof}
		From theorem \ref{thm:valueFunctionDecomposition} and \eqref{eq:DecomposedCGMVF}, we have
		\begin{equation}
		\label{eq:DecomposedCountMACompatiblePolicyGrad}
		\nabla_\theta V_1(\pi) = \sum_t\mathbb{E}_{\bs{s}_t, \bs{a}_t}\Big[\sum_{i,j} n_{t}(i,j)\frac{\partial  \log \pi\big(j|i, o(i, \bs{\n}_{\bs{s}_t})\big)}{\partial \theta} f_{w}(i, j, o(i, \bs{\n}_{\bs{s}_t}))\Big]
		\end{equation}  
		We can expand the above expression as:
		\begin{align}
		\nabla_\theta V_1(\pi)=\sum_{\bs{s}_{1:H},\bs{a}_{1:H}} P(\bs{s}_{1:H}, \bs{a}_{1:H}) \Big[\sum_{t=1}^H\sum_{i,j} n_{t}(i,j)\frac{\partial  \log \pi\big(j|i, o(i, \bs{\n}_{\bs{s}_t})\big)}{\partial \theta} f_{w}(i, j, o(i, \bs{\n}_{\bs{s}_t}))\Big]\nonumber
		\end{align}
		From~\cite{NguyenKL17}, we know that the probability $ P(\bs{s}_{1:H}, \bs{a}_{1:H})$ depends only on counts $\bs{\n}$ generated by the joint-state and action trajectory $(\bs{s}_{1:H},\bs{a}_{1:H})$ and is equal to $f(\bs{\n}_{1:T})$ in~\eqref{eq:joint}. Using this result, we have:
		\begin{align}
		\nabla_\theta V_1(\pi) = \sum_{\bs{s}_{1:H},\bs{a}_{1:H}} f(\bs{\n}_{1:H}) \Big[\sum_{t=1}^H\sum_{i,j} n_{t}(i,j)\frac{\partial  \log \pi\big(j|i, o(i, \bs{\n}_{\bs{s}_t})\big)}{\partial \theta} f_{w}(i, j, o(i, \bs{\n}_{\bs{s}_t}))\Big]\nonumber
		\end{align}
		Notice that the entire expression inside the summation above depends only on the resulting counts $\bs{\n}_{1:H}$. We also know from~\cite{NguyenKL17} that $h(\bs{\n}_{1:H})$ in~\eqref{eq:perm} counts the total number of ordered $M$ state-action trajectories with sufficient statistic equal to $\bs{\n}_{1:H}$. Therefore, we can replace the summation over $(\bs{s}_{1:H},\bs{a}_{1:H})$ by summation over all the possible valid counts $\bs{\n}_{1:H}\in \Omega_{1:H}$ and multiply the inner expression by $h(\cdot)$ to get:
		\begin{align}
		\nabla_\theta V_1(\pi) &= \sum_{\bs{\n}_{1:H}\in \Omega_{1:H}}h(\bs{\n}_{1:H}) f(\bs{\n}_{1:H}) \Big[\sum_{t=1}^H\sum_{i,j} n_{t}(i,j)\frac{\partial  \log \pi\big(j|i, o(i, \bs{\n}_{t})\big)}{\partial \theta} f_{w}(i, j, o(i, \bs{\n}_{t}))\Big]\nonumber \\
		&= \sum_{\bs{\n}_{1:H}\in \Omega_{1:H}} P(\bs{\n}_{1:H}) \Big[\sum_{t=1}^H\sum_{i,j} n_{t}(i,j)\frac{\partial  \log \pi\big(j|i, o(i, \bs{\n}_{t})\big)}{\partial \theta} f_{w}(i, j, o(i, \bs{\n}_{t}))\Big]
		\end{align}
		The above equation proves the theorem.
	\end{proof}
	
	\section{Training with individual value function}
	Recall from \cite{NguyenKL17} that in fictitious EM, for each sample $ \bn^{\xi} $ of the count, the individual value function function is computed as
	\begin{align}
	&V^{\xi}_H(i, j) \= r_H(i, j, n^{\xi}_{H}(i)) \label{eq:dp1}\\
	&V^{\xi}_t(i, j) \= r_t(i, j, n^{\xi}_{t}(i)) \!+\!\hspace{-10pt} \sum_{i'\in S, j'\in A}\hspace{-10pt}  \phi^{\bs{n}^\xi}_t(i'|i,j)  \pi_{t+1}^{\bs{n}^\xi}(j'|i') V^{\xi}_{t+1}(i', j') 
	\end{align}
	with
	\begin{align}
	\phi^{\bs{n}^\xi}_t(i' | i,  j) &= \frac{n^\xi_{t}(i, j, i')}{n^\xi_{t}(i, j)}; \;\pi_t^{\bs{n}^\xi}(j|i) = \frac{n^\xi_{t}(i, j)}{n^\xi_{t}(i)} \\
	P_1^{\bs{n}^\xi}( i) &= \frac{n^\xi_{1}(i)}{M}; \; r_t^{\bs{n}^\xi}( i,  j) \= r_t(i, j, n^\xi_{t}(i))
	\end{align}
	We denote the total accumulated reward from time $ t $ to $ H $ of count samples $ \bn^{\xi} $ to be
	\begin{equation}
	R^{\xi}_t = \sum_{t'=t}^H\sum_{i,j} n^{\xi}_{t'}(i,j)r_{t'}(i,j, n^{\xi}_{t'}(i))
	\end{equation}
	
	\begin{lemma}
		\label{lem:local}
		The empirical return $R_t^\xi$ for the time step $t$ given the count sample  $\bs{\n}^\xi_{1:H}$ can be re-parameterized as:
		
		\begin{equation}
		\label{eq:decomposedR}
		R^{\xi}_t  = \sum_{i\in S, j\in A} n_t^\xi(i, j) V_t^\xi(i, j)
		\end{equation}
	\end{lemma}
	
	\begin{proof}
		We know from~\cite{NguyenKL17} that the individual value function $V_t^{\xi}$ for a count sample $\bs{\n}^\xi$ is given by the following expectation:
		\begin{align}
		V_t^{\xi}(i, j) = \mathbb{E}\Big[\sum_{t'=t}^H r_{t'}^m | s_t^m=i, a_t^m=j, \bs{\n}_{1:H}^\xi\Big]
		\end{align}
		By definition, the total empirical return $R^{\xi}_t$ is given by the summation of individual value function for all the agents $m$:
		\begin{align}
		R^{\xi}_t &= \sum_{m} V_t^{\xi}(s_t^m, a_t^m) \\
		&= \sum_{i\in S, j\in A}  n_t^\xi(i, j) V_t^\xi(i, j)
		\end{align}
		For the last equation, we have used the fact that agents which are in the same state $i$ and and take the same action $j$, they have the same value function (as all the agents are identical).
	\end{proof}
	\section{Experimental setup}
	To optimize the policy and value function network, we use Adam optimizer with the learning rate chosen from $\{10^{-4}, 10^{-3}, 10^{-2}\}$ for the best performance of algorithms. As observation of the count can have different magnitude in grid navigation and taxi domain, we use batch normalization \cite{ioffe2015batch} for all the networks. To address the different magnitude of rewards, i.e. the grid navigation having maximum reward 1 and taxi domain having maximum reward 100, we normalize the empirical return by adaptively rescaling targets method as in \cite{van2016learning}.\\
	For actor-critic update, we consider the batch size to be 100 for grid navigation and 48 for taxi navigation.\\
	For 'o0' and 'o1' cases, we use no hidden layer in the network. For 'oN' case, we use 2 hidden layers with size $18\times 18$ for both policy and value function network. We use relu unit for all hidden layers and softmax unit for output of policy and linear output for value function.	
\end{document}